\documentclass[11pt,letterpaper]{article}
\usepackage[letterpaper]{geometry}
\pdfoutput=1
\usepackage{acl2012}
\usepackage{times}
\usepackage{latexsym}
\usepackage{amsmath}
\usepackage{multirow}
\usepackage{url}

\usepackage[style=authoryear-icomp,backref,sortcites=true,sorting=ynt,sortcase=false,backend=biber,maxcitenames=2,maxbibnames=999]{biblatex}
\bibliography{ms}

\usepackage{amsthm}
\usepackage{xspace}
\usepackage{graphicx}
\usepackage[normalem]{ulem}
\usepackage{pdflscape}
\usepackage{rotating}
\usepackage{color}
\usepackage{algorithm}
\usepackage{enumitem}
\usepackage{synttree}
\usepackage{array}
\usepackage{arydshln}
\usepackage{xcolor,colortbl}
\usepackage{multicol}
\usepackage{framed}
\usepackage{tikz}
\usepackage{pgf}
\usepackage{tikz-qtree}
\usepackage{subcaption}
\usepackage{mdwlist}
\usepackage{makeidx}
\usepackage{fix-cm}
\usepackage{etoolbox}
\usepackage{mdframed}
\usepackage{comment}
\usepackage{wrapfig}
\usepackage{accents}
\usepackage{pdfpages}

\usepackage[hidelinks]{hyperref}

\DeclareNameAlias{author}{first-last}

\hypersetup{colorlinks=true,
plainpages=false,
linktocpage=true,
linkcolor=blue,
menucolor=blue,
anchorcolor=black,
citecolor=gray,
urlcolor=gray,
pdftitle={Parsing with Traces: An O(n4) Algorithm and a Structural Representation},
pdfauthor={Jonathan K. Kummerfeld and Dan Klein},
pdfsubject={Natural Language Processing},
pdfpagelayout=OneColumn}

\usetikzlibrary{arrows.meta,backgrounds,positioning,fit,calc,intersections,shapes,patterns}

\tikzset{%
  not allowed/.style={%
    densely dotted,
    thick,
    color=red
  }
}
\tikzset{%
  required directed/.style={%
    shorten >=1pt,
    >=Stealth,
    ->
  }
}
\tikzset{%
  required undirected/.style={%
  }
}

\definecolor{mygrey}{rgb}{0.8,0.8,0.8}

\newcommand*\underdot[1]{\underaccent{\dot}{#1}}

\newcommand{\ptb}{PTB\xspace}

\newcommand{\myeg}{e.g.\@\xspace}
\newcommand{\myie}{i.e.\@\xspace}
\newcommand{\myvs}{vs.\@\xspace}
\newcommand{\oneEC}{1-EC\@\xspace}
\newcommand{\problemStruct}{Locked-Chain\@\xspace}

\newcommand{\pencilPoint}{$\mathcal{P}t$}
\newcommand{\pencilPointMath}{\mathcal{P}t}
\newcommand{\vertexCovMath}{covered}
\newcommand{\vertexVisMath}{visible}

\newcommand{\vertexVis}{visible\@\xspace}
\newcommand{\tightparagraph}[1]{\textbf{#1:}}

\makeatletter
\renewenvironment{proof}[1][\proofname]{\par
  \vspace{-0.5\topsep}
  \pushQED{\qed}%
  \normalfont
  \topsep0pt \partopsep0pt 
  \trivlist
  \item[\hskip\labelsep
        \itshape
    #1\@addpunct{.}]\ignorespaces
}{%
  \popQED\endtrivlist\@endpefalse
  \addvspace{3pt plus 6pt} 
}
\makeatother

\newmdenv[bottomline=false,rightline=false]{lefttopbot}
\mdfdefinestyle{mdftight}{skipabove=0pt,innerleftmargin=3pt,innerrightmargin=0pt,innertopmargin=0pt,innerbottommargin=0pt}

\newcolumntype{x}[1]{%
>{\raggedleft\hspace{0pt}}p{#1}}%

\branchheight{0.33in}
\trianglebalance{50}
\childsidesep{2ex}
\childattachsep{2ex}

\newtheorem{lemma}{Lemma}
\newtheorem{theorem}{Theorem}
\newtheorem{definition}{Definition}

\newcommand{\captionfonts}{\small}
\makeatletter
\long\def\@makecaption#1#2{%
  \vskip\abovecaptionskip
  \sbox\@tempboxa{{\captionfonts #1: #2}}%
  \ifdim \wd\@tempboxa >\hsize
    {\captionfonts #1: #2\par}
  \else
    \hbox to\hsize{\hfil\box\@tempboxa\hfil}%
  \fi
  \vskip\belowcaptionskip}
\makeatother

\renewbibmacro{in:}{\ifentrytype{article}{}{\printtext{In }}}

\DeclareFieldFormat[article,inbook,incollection,inproceedings,patent,unpublished]{citetitle}{#1}
\DeclareFieldFormat[article,inbook,incollection,inproceedings,patent,unpublished]{title}{#1}
\DeclareFieldFormat[thesis]{citetitle}{\mkbibemph{#1}}
\DeclareFieldFormat[thesis]{title}{\mkbibemph{#1}}
\DeclareFieldFormat{url}{\url{#1}}

\title{Parsing with Traces: An $\boldsymbol{O(n^4)}$ Algorithm and a Structural Representation}

\author{
  Jonathan K. Kummerfeld \and Dan Klein \\
  Computer Science Division \\
  University of California, Berkeley \\
  Berkeley, CA 94720, USA \\
  {\tt \{jkk,klein\}@cs.berkeley.edu}
}

\date{}

\begin{document}

\maketitle

\citetrackerfalse\pagetrackerfalse\backtrackerfalse

%
%

\begin{abstract}
  General treebank analyses are graph structured, but parsers are typically restricted to tree structures for efficiency and modeling reasons.
  We propose a new representation and algorithm for a class of graph structures that is flexible enough to cover almost all treebank structures, while still admitting efficient learning and inference.
  In particular, we consider directed, acyclic, one-endpoint-crossing graph structures, which cover most long-distance dislocation, shared argumentation, and similar tree-violating linguistic phenomena.
  We describe how to convert phrase structure parses, including traces, to our new representation in a reversible manner.
  Our dynamic program uniquely decomposes structures, is sound and complete, and covers $97.3\%$ of the Penn English Treebank.
  We also implement a proof-of-concept parser that recovers a range of null elements and trace types.
\end{abstract}

\section{Introduction}

Many syntactic representations use graphs and/or discontinuous structures, such as traces in Government and Binding theory and f-structure in Lexical Functional Grammar \parencite{gb,Bresnan:1982}.
Sentences in the Penn Treebank \parencite[\ptb,][]{ptb} have a core projective tree structure and trace edges that represent control structures, wh-movement and more.
However, most parsers and the standard evaluation metric ignore these edges and all null elements.
By leaving out parts of the structure, they fail to provide key relations to downstream tasks such as question answering.
While there has been work on capturing some parts of this extra structure, it has generally either been through post-processing on trees
\parencite{Johnson:2002,Jijkoun:2003,Campbell:2004,Levy:2004,Gabbard:2006}
or has only captured a limited set of phenomena via grammar augmentation
\parencite{collins:1997,dienes-dubey:2003,schmid:2006,cai-chiang-goldberg:2011}.

We propose a new general-purpose parsing algorithm that can efficiently search over a wide range of syntactic phenomena.
Our algorithm extends a non-projective tree parsing algorithm \parencite{ec,ec-gp} to graph structures, with improvements to avoid derivational ambiguity while maintaining an $O(n^4)$ runtime.
Our algorithm also includes an optional extension to ensure parses contain a directed projective tree of non-trace edges.

Our algorithm cannot apply directly to constituency parses--it requires lexicalized structures similar to dependency parses.
We extend and improve previous work on lexicalized constituent representations \parencite{cck,Shen:2007,hayashi-nagata:2016} to handle traces.
In this form, traces can create problematic structures such as directed cycles, but we show how careful choice of head rules can minimize such issues.

We implement a proof-of-concept parser, scoring $88.1$ on trees in section 23 and $70.6$ on traces.
Together, our representation and algorithm cover $97.3\%$ of sentences, far above the coverage of projective tree parsers ($43.9\%$).

\section{Background}

This work builds on two areas: non-projective tree parsing, and parsing with null elements.

\textbf{Non-projectivity} is important in syntax for representing many structures, but inference over the space of all non-projective graphs is intractable.
Fortunately, in practice almost all parses are covered by well-defined subsets of this space.
For dependency parsing, recent work has defined algorithms for inference within various subspaces \parencite{Gomez-Rodriguez:2010,ec}.
We build upon these algorithms and adapt them to constituency parsing.
For constituency parsing, a range of formalisms have been developed that are mildly-context sensitive, such as CCG \parencite{Steedman:2000}, LFG \parencite{Bresnan:1982}, and LTAG \parencite{Joshi:1997}.

Concurrently with this work, \textcite{Cao-etal:2017:ACL} also proposed a graph version of \textcite{ec}'s One-Endpoint Crossing (\oneEC) algorithm.
However, Cao's algorithm does not consider the direction of edges\footnote{
  To produce directed edges, their parser treats the direction as part of the edge label.
} and so it could produce cycles, or graphs with multiple root nodes.
Their algorithm also has spurious ambiguity, with multiple derivations of the same parse structure permitted.
One advantage of their algorithm is that by introducing a new item type it can handle some cases of the \problemStruct we define below (specifically, when $N$ is even), though in practise they also restrict their algorithm to ignore such cases.
They also show that the class of graphs they generate corresponds to the \oneEC pagenumber-2 space, a property that applies to this work as well\footnote{
  This is a topological space with two half-planes sharing a boundary.
  All edges are drawn on one of the two half-planes and each half-plane contains no crossings.
}.

\textbf{Parsing with Null Elements} in the \ptb has taken two general approaches.
The first broadly effective system was \textcite{Johnson:2002}, which post-processed the output of a parser, inserting extra elements.
This was effective for some types of structure, such as null complementizers, but had difficulty with long distance dependencies.
The other common approach has been to thread a trace through the tree structure on the non-terminal symbols.
\textcite{collins:1997}'s third model used this approach to recover wh-traces, while \textcite{cai-chiang-goldberg:2011} used it to recover null pronouns, and others have used it for a range of movement types \parencite{dienes-dubey:2003,schmid:2006}.
These approaches have the disadvantage that each additional trace dramatically expands the grammar.

Our representation is similar to LTAG-Spinal \parencite{Shen:2007} but has the advantage that it can be converted back into the \ptb representation.
\textcite{hayashi-nagata:2016} also incorporated null elements into a spinal structure but did not include a representation of co-indexation.
In related work, dependency parsers have been used to assist in constituency parsing, with varying degrees of representation design, but only for trees \parencite{hall2007hybrid,hall-nivre:2008:PaGe,kong-rush-smith:2015:NAACL-HLT,fernandezgonzalez-martins:2015:ACL-IJCNLP}.

\textcite{kato-matsubara:2016} described a new approach, modifying a transition-based parser to recover null elements and traces, with strong results, but using heuristics to determine trace referents.

\section{Algorithm} \label{sec:overall-algorithm}

\tikzset{%
  not allowed/.style={%
    dotted,
    very thick,
    color=red
  }
}
\tikzset{%
  required/.style={%
  }
}
\tikzset{%
  optional/.style={%
    dashed
  }
}
\tikzset{%
  pointO/.style={%
    fill=black,regular polygon, regular polygon sides=4,inner sep=1pt
  }
}
\newlength\vertSmall
\newlength\vertBig
\newlength\vertBigger
\newlength\labelGap
\newlength\coordGap

Our algorithm is a dynamic program, similar at a high level to CKY \parencite{Cocke:1969,Kasami:1966,Younger:1967}.
The states of our dynamic program (\emph{items}) represent partial parses.
Usually in CKY, items are defined as covering the $n$ words in a sentence, starting and ending at the spaces between words.
We follow \textcite{eisner:1996}, defining items as covering the $n{-}1$ spaces in a sentence, starting and ending on words, as shown in Figure~\ref{fig:alg-example}.
This means that we process each word's left and right dependents separately, then combine the two halves.

We use three types of items:
(1) a single \emph{edge}, linking two words,
(2) a continuous \emph{span}, going from one word to another, representing all edges linking pairs of words within the span,
(3) a span (as defined in 2) plus an additional word outside the span, enabling the inclusion of edges between that word and words in the span.

Within the CKY framework, the key to defining our algorithm is a set of rules that specify which items are allowed to combine.
From a bottom-up perspective, a parse is built in a series of steps, which come in three types:
(1) adding an edge to an item,
(2) combining two items that have non-overlapping adjacent spans to produce a new item with a larger span,
(3) combining three items, similarly to (2).

\tightparagraph{Example}
To build intuition for the algorithm, we will describe the derivation in Figure~\ref{fig:alg-example}.
Note, item sub-types (I, X, and N) are defined below, and included here for completeness.

\begin{figure}
\centering
\scalebox{0.8}{ \tikzset{%
  leftParent/.style={%
    ->,
    >=Stealth,
    shorten >=1pt,
    thin
  }
}%
\tikzset{%
  rightParent/.style={%
    <-,
    >=Stealth,
    shorten <=1pt,
    thin
  }
}%
\tikzset{%
  len1/.style={%
    out=30,
    in=150
  }
}%
\tikzset{%
  len2/.style={%
    out=32,
    in=148
  }
}%
\tikzset{%
  len3/.style={%
    out=34,
    in=146
  }
}%
\tikzset{%
  len4/.style={%
    out=34,
    in=146
  }
}%
\tikzset{%
  extPoint/.style={%
    fill=black,regular polygon, regular polygon sides=4,inner sep=0.75pt
  }
}%
\tikzset{%
  myGuide/.style={%
    densely dotted,
    thick,
    color=black!25
  }
}%
\begin{tikzpicture}%
  \pgfmathsetlength{\vertSmall}{4ex}%
  \pgfmathsetlength{\vertBig}{6ex}%
  \coordinate (offset) at (0.3, 0);%
  \coordinate (voffset0) at (0, 0);%
  \coordinate (voffset1) at (0, 0.1);%
  \node (v0) at (0, 0) {};
  \node (w0) at (0, 0) {};
  \node (w1) at (2, 0) {};
  \node (w2) at (4, 0) {};
  \node (w3) at (6, 0) {};
  \node (w4) at (8, 0) {};

  \node (vA) [below=\vertBig of v0] {};
  \node (v1) [above=\vertSmall of v0] {};
  \node (v2) [above=\vertSmall of v1] {};
  \node (v3) [above=\vertSmall of v2] {};
  \node (v4) [above=\vertSmall of v3] {};
  \node (v5) [above=\vertSmall of v4] {};

  \draw [myGuide] (w0 |- vA) -- (w0 |- v5);
  \draw [myGuide] (w1 |- vA) -- (w1 |- v3);
  \draw [myGuide] (w2 |- vA) -- (w2 |- v5);
  \draw [myGuide] (w3 |- vA) -- (w3 |- v5);
  \draw [myGuide] (w4 |- vA) -- (w4 |- v5);

  \node (w0text) [below=\vertBig of w0] {\strut ROOT};
  \node (w1text) [below=\vertBig of w1] {\strut We};
  \node (w2text) [below=\vertBig of w2] {\strut like};
  \node (w3text) [below=\vertBig of w3] {\strut running};
  \node (w4text) [below=\vertBig of w4] {\strut .};

  \draw [leftParent,len2] ($(w0 |- vA) + (offset)$) to ($(w2 |- vA)$);
  \draw [rightParent,len1] ($(w1 |- vA) + (offset)$) to ($(w2 |- vA) - (offset) + (voffset0)$);
  \draw [rightParent,len2] ($(w1 |- vA)$) to ($(w3 |- vA)$);
  \draw [leftParent,len1] ($(w2 |- vA) + (offset)$) to ($(w3 |- vA) - (offset) + (voffset0)$);
  \draw [leftParent,len2] ($(w2 |- vA) + (voffset0)$) to ($(w4 |- vA) - (offset)$);

  \draw ($(w0) + (offset)$) -- node[at start,below=-2pt] {\small $\;$ I$_{0,1}$} ($(w1) - (offset)$);
  \draw ($(w1) + (offset)$) -- node[at start,below=-2pt] {\small $\;$ I$_{1,2}$} ($(w2) - (offset)$);
  \draw ($(w2) + (offset)$) -- node[at start,below=-2pt] {\small $\;$ I$_{2,3}$} ($(w3) - (offset)$);
  \draw ($(w3) + (offset)$) -- node[at start,below=-2pt] {\small $\;$ I$_{3,4}$} ($(w4) - (offset)$);
  \node [anchor=east] (step0) at (w0 |- v0) {\strut \small (1)};

  \draw ($(w1 |- v1) + (offset)$) -- node[at start,below=-2pt] {\small $\;$ I$_{1,2}$} ($(w2 |- v1) - (offset)$);
  \draw ($(w2 |- v1) + (offset)$) -- node[at start,below=-2pt] {\small $\;$ I$_{2,3}$} ($(w3 |- v1) - (offset)$);
  \draw ($(w3 |- v1) + (offset)$) -- node[at start,below=-2pt] {\small $\;\;$ X$_{3,4,2}$} ($(w4 |- v1) - (offset)$);
  \draw [rightParent,len1] ($(w1 |- v1) + (offset)$) to ($(w2 |- v1) - (offset) + (voffset0)$);
  \draw [leftParent,len1] ($(w2 |- v1) + (offset) + (voffset0)$) to ($(w3 |- v1) - (offset)$);
  \node [extPoint] at (w2 |- v1) {};
  \draw [leftParent,len2] ($(w2 |- v1)$) to ($(w4 |- v1) - (offset)$);
  \node [anchor=east] (step0) at (w0 |- v1) {\strut \small (2)};

  \draw ($(w1 |- v2) + (offset)$) -- node[at start,below=-2pt] {\small $\;\;$ X$_{1,2,3}$} ($(w2 |- v2) - (offset)$);
  \draw [rightParent,len1] ($(w1 |- v2) + (offset)$) to ($(w2 |- v2) - (offset) + (voffset0)$);
  \draw [rightParent,len2] ($(w1 |- v2) + (offset) + (voffset1)$) to ($(w3 |- v2) + (voffset0)$);
  \node [extPoint] at (w3 |- v2) {};
  \node [anchor=east] (step0) at (w0 |- v2) {\strut \small (3)};

  \draw ($(w0 |- v3) + (offset)$) -- node[at start,below=-2pt] {\small $\;\;$ N$_{0,2,3}$} ($(w2 |- v3) - (offset)$);
  \draw [rightParent,len1] ($(w1 |- v3) + (offset)$) to ($(w2 |- v3) - (offset) + (voffset0)$);
  \draw [rightParent,len2] ($(w1 |- v3)$) to ($(w3 |- v3) + (voffset0)$);
  \node [extPoint] at (w3 |- v3) {};
  \node [anchor=east] (step0) at (w0 |- v3) {\strut \small (4)};

  \draw ($(w0 |- v4) + (offset)$) -- node[at start,below=-2pt] {\small $\;\;$ N$_{0,2,3}$} ($(w2 |- v4) - (offset)$);
  \draw [leftParent,len2] ($(w0 |- v4) + (offset)$) to ($(w2 |- v4) - (offset) + (voffset1)$);
  \draw [rightParent,len1] ($(w1 |- v4) + (offset)$) to ($(w2 |- v4) - (offset) + (voffset0)$);
  \draw [rightParent,len2] ($(w1 |- v4)$) to ($(w3 |- v4) + (voffset0)$);
  \node [extPoint] at (w3 |- v4) {};
  \node [anchor=east] (step0) at (w0 |- v4) {\strut \small (5)};

  \draw ($(w0 |- v5) + (offset)$) -- node[at start,below=-2pt] {\small $\;\;$ I$_{0,4}$} ($(w4 |- v5) - (offset)$);
  \draw [leftParent,len2] ($(w0 |- v5) + (offset)$) to ($(w2 |- v5)$);
  \draw [rightParent,len1] ($(w1 |- v5) + (offset)$) to ($(w2 |- v5) - (offset) + (voffset0)$);
  \draw [rightParent,len2] ($(w1 |- v5)$) to ($(w3 |- v5) + (voffset0)$);
  \draw [leftParent,len1] ($(w2 |- v5) + (offset) + (voffset0)$) to ($(w3 |- v5) - (offset)$);
  \draw [leftParent,len2] ($(w2 |- v5)$) to ($(w4 |- v5) - (offset)$);
  \node [anchor=east] (step0) at (w0 |- v5) {\strut \small (6)};

\end{tikzpicture}%
 }%
\vspace{-5mm}
\caption{\label{fig:alg-example}
An example derivation using our graph parsing deduction rules.
}
\end{figure}
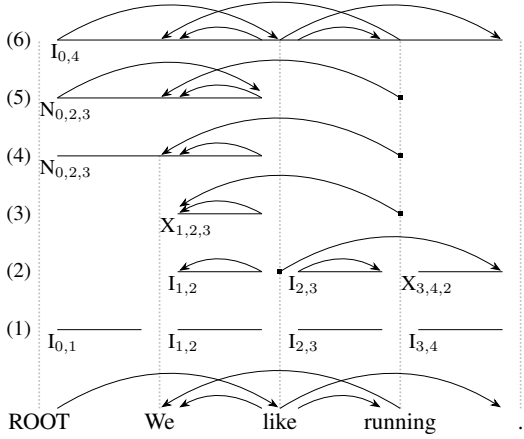

\noindent
(1) We initialize with spans of width one, going between adjacent words, \myeg between \emph{ROOT} and \emph{We}. \\
\strut\hfill $\emptyset \; \mapsto \; I_{0,1}$\hfill\strut

\noindent
(2) Edges can be introduced in exactly two ways, either by linking the two ends of a span, \myeg \emph{like}--\emph{running}, or by linking one end of a span with a word outside the span, \myeg \emph{like}--.\@\xspace (which in this case forms a new item that has a span and an external word). \\
\strut\hfill$I_{2,3} \; \land \; like$--$running \; \mapsto \; I_{2,3}$\hfill\strut \\
\strut\hfill$I_{3,4} \; \land \; like$--.$ \; \mapsto \; X_{3,4,2}$\hfill\strut

\noindent
(3) We add a second edge to one of the items. \\
\strut\hfill$I_{1,2} \; \land \; running$--$We \; \mapsto \; X_{1,2,3}$\hfill\strut

\noindent
(4) Now that all the edges to \emph{We} have been added, the two items either side of it are combined to form an item that covers it. \\
\strut\hfill$I_{0,1} \; \land \; X_{1,2,3} \; \mapsto \; N_{0,2,3}$\hfill\strut

\noindent
(5) We add an edge, creating a crossing because \emph{We} is an argument of a word to the right of \emph{like}. \\
\strut\hfill$N_{0,2,3} \; \land \; ROOT$--$like \; \mapsto \; N_{0,2,3}$\hfill\strut

\noindent
(7) We use a ternary rule to combine three adjacent items.
In the process we create another crossing. \\
\strut\hfill$N_{0,2,3} \; \land \; I_{2,3} \; \land \; X_{3,4,2} \; \mapsto \; I_{0,6}$\hfill\strut

\subsection{Algorithm definition}

\paragraph{Notation}

Vertices are $p$, $q$, etc.
Continuous ranges are $[pq]$, $[pq)$, $(pq]$, or $(pq)$, where the brackets indicate inclusion, $[\,]$, or exclusion, $(\,)$, of each endpoint.
A span $[pq]$ and vertex $o$ that are part of the same item are $[pq.o]$.
Two vertices and an arrow indicate an edge, $\vec{pq}$.
Two vertices without an arrow are an edge in either direction, $pq$.
Ranges and/or vertices connected by a dash define a set of edges, \myeg the set of edges between $o$ and $(pq)$ is $o$--$(pq)$ (in some places we will also use this to refer to an edge from the set, rather than the whole set).
If there is a path from $p$ to $q$, $q$ is \textit{reachable} from $p$.

\paragraph{Item Types}

As shown in Figure~\ref{fig:alg-example}, our items start and end on words, fully covering the spaces in between.
Earlier we described three item types: an edge, a span, and a span plus an external vertex.
Here we define spans more precisely as $I$, and divide the span plus an external point case into five types differing in the type of edge crossing they contain:

\begingroup
\setlength{\columnsep}{6pt}%
\setlength{\intextsep}{3pt}%

\begin{wrapfigure}[2]{r}{0pt}
  \begin{tikzpicture}
    \node (p) at (0, 0) {$p$};
    \node (q) at (1, 0) {$q$};
    \draw (p.north) -- (q.north);
  \end{tikzpicture}
\end{wrapfigure}
\noindent \textbf{$I$, Interval} A span for which there are no edges $sr : r \in (pq)$ and $s \notin [pq]$.

\noindent
\begin{wrapfigure}[2]{r}{0pt}
  \begin{tikzpicture}
    \node (p) at (0, 0) {\phantom{$o$}};
    \node (q) at (1.0, 0) {\phantom{$o$}};
    \node (o) at (1.5, 0) {$o$};
    \node [pointO] at (o.north) {};
    \draw (p.north) -- (q.north);
    \draw [out=45,in=135] (p.north) to (o.north);
  \end{tikzpicture}
\end{wrapfigure}
\noindent \textbf{$X$, Exterval} An interval and either $op$ or $oq$, where $o \notin [pq]$.

\noindent
\begin{wrapfigure}[2]{r}{0pt}
  \begin{tikzpicture}
    \node (p) at (0, 0) {};
    \node (m1) at (0.3, 0) {};
    \node (m2) at (0.6, 0) {};
    \node (m3) at (0.9, 0) {};
    \node (m4) at (1.2, 0) {};
    \node (q) at (1.5, 0) {};
    \node (o) at (2, 0) {};
    \draw (p.center) -- (q.center);
    \node [pointO] at (o.center) {};
    \draw [out=45,in=135] (m1.center) to (o.center);
    \draw [out=45,in=135] (m4.center) to (o.center);
    \draw [out=45,in=135] (p.center) to (m2.center);
    \draw [out=45,in=135] (m3.center) to (q.center);
  \end{tikzpicture}
\end{wrapfigure}
\noindent \textbf{$B$, Both} A span and vertex $[pq.o]$, for which there are no edges $sr : r \in (pq)$ and $s \notin [pq] \cup o$.
Edges $o$--$[pq]$ may be crossed by $pq$, $p$--$(pq)$ or $q$--$(pq)$, and at least one crossing of the second and third types occurs.
Edges $o$--$(pq)$ may not be crossed by $(pq)$--$(pq)$ edges.

\begin{wrapfigure}[3]{r}{0pt}
  \begin{tikzpicture}
    \node (p) at (0, 0) {};
    \node (m) at (0.5, 0) {};
    \node (m2) at (1.0, 0) {};
    \node (q) at (1.5, 0) {};
    \node (o) at (2, 0) {};
    \draw (p.center) -- (q.center);
    \node [pointO] at (o.center) {};
    \draw [out=45,in=135] (m.center) to (o.center);
    \draw [out=45,in=135] (p.center) to (m2.center);

    \node (rp) at (0, -0.6) {};
    \node (rm) at (0.5, -0.6) {};
    \node (rm2) at (1.0, -0.6) {};
    \node (rq) at (1.5, -0.6) {};
    \node (ro) at (2, -0.6) {};
    \draw (rp.center) -- (rq.center);
    \node [pointO] at (ro.center) {};
    \draw [out=45,in=135] (rm2.center) to (ro.center);
    \draw [out=45,in=135] (rm.center) to (rq.center);
  \end{tikzpicture}
\end{wrapfigure}
\noindent \textbf{$L$, Left} Same as $B$, but $o$--$(pq)$ edges may only cross $p$--$(pq]$ edges. \\
\noindent \textbf{$R$, Right} Symmetric with $L$.

\begin{wrapfigure}[2]{r}{0pt}
  \begin{tikzpicture}
    \node (p) at (0, 0) {};
    \node (m) at (0.5, 0) {};
    \node (q) at (1.5, 0) {};
    \node (o) at (2, 0) {};
    \draw (p.center) -- (q.center);
    \node [pointO] at (o.center) {};
    \draw [out=45,in=135] (m.center) to (o.center);
  \end{tikzpicture}
\end{wrapfigure}
\noindent \textbf{$N$, Neither} An interval and a vertex $[pq.o]$, with at least one $o$--$(pq)$ edge, which can be crossed by $pq$, but no other $[pq]$--$[pq]$ edges.

\endgroup

Items are further specified as described in Alg.~\ref{alg:rules}.
Most importantly, for each pair of $o$, $p$, and $q$ in an item, the rules specify whether one is a parent of the other, and if they are directly linked by an edge.

For an item $H$ with span $[ij]$, define $\vertexCovMath(H)$ as $(ij)$, and define $\vertexVisMath(H)$ as $\{i, j\}$.
When an external vertex $x$ is present, it is in $\vertexVisMath(H)$.
Also, call the union of multiple such sets $\vertexCovMath(F, G, H)$, and $\vertexVisMath(F, G, H)$.

\begin{algorithm*}
\footnotesize

\uline{Adding Edges}: Consider a span $[lr]$ and vertex $x \notin [lr]$. \\
Edges between $l$ and $r$ can be added to items $I$, $N$, $L$, $R$, and $B$ (making $\hat{L}$ and $\hat{N}$ in those cases). \\
Edges between $l$ and $x$ can be added to items $I$ (forming an $X$), $R$, and $N$. \\
Edges between $r$ and $x$ can be added to items $I$ (forming an $X$), $L$, and $N$. \\
The $l$--$r$ edge cannot be added after another edge, and $N$ items cannot get both $l$--$x$ and $r$--$x$ edges.

\uline{Combining Items}: In the rules below the following notation is used: \\
For this explanation items are $T[lr \; c_{rl} \; c_{lr}]$ and $T[lrx \; c_{rl} \; c_{xl} \; c_{lr} \; c_{xr} \; c_{lx} \; c_{rx}]$. \\
$T$ is the type of item. Multiple letters indicate any of those types are allowed. \\
For the next three types of notation, if an item does not have a mark, either option is valid. \\
$\underdot{T}$ and $\uwave{T}$ indicate the number of edges between the external vertex and the span: one or more than one respectively. \\
\textbf{$\cdotp$}$T$ and $T$\textbf{$\cdotp$} indicate the position of the external vertex relative to the item's span (left or right respectively). \\
$\hat{T}$ indicates for $N$ and $L$ that $\forall p \in (ij) \exists rs : i{\le}r{<}p{<}s{\le}j$. In (11) and (12) it is optional, but true for output iff true for input. \\
$l$, $r$, and $x$: the position of the left end of the span, the right end, and the external vertex, respectively. \\ 
$c_{rl}$, $c_{xl}$, etc: connectivity of each pair of \vertexVis vertices, from the first subscript to the second.
Using $c_{rl}$ as an example, these can be $.$ (unconstrained), $d$ ($\vec{rl}$ must exist), $p$ ($l$ is reachable from $r$, but $\vec{rl}$ does not exist), $n$ ($l$ is not reachable from $r$), $\overline{d}$ ($= p \lor n$), $\overline{n}$~($= d \lor p$).
Note: In the generated rules every value is $d$, $p$, or $n$, leading to multiple rules per template below.

\vspace{-8mm}
\begin{multicols}{2}
\footnotesize
\begin{flalign*}
& \begin{array}{l}
  I[ij \;\; n\overline{d}] \leftarrow \max \\
  \left\{
    \begin{array}{l}
      \begin{array}{l}
        (\mathrm{Init}) \;\;\; j=i{+}1 \\
        (1) \;\;\; I[i \;\; i{+}1 \;\; nn] \quad I[i{+}1 \;\; j \;\; \overline{n}n] \\
      \end{array} \\
      \max_{k \in (i, j)} \\
      \left\{
        \begin{array}{l}
          (2) \;\;\; I[ik \;\; nd] \quad I[kj \;\; ..] \\
          (3) \;\;\; BLRN\cdotp [ikj \;\; nnd\overline{d}\overline{dd}] \quad I[kj \;\; ..] \\
          \max_{l \in (k, j)} \\
          \left\{
            \begin{array}{l}
              (4) \;\;\; RN\cdotp [ikl \;\; nnd\overline{d}\overline{dd}] \quad I[kl \;\; ..] \quad \cdotp LNX[ljk \;\; .\overline{d}..\overline{d}.] \\
              (5) \;\;\; BLRN\cdotp [ikl \;\; nnd\overline{d}\overline{dd}] \quad I[kl \;\; ..] \quad I[lj \;\; ..] \\
            \end{array}
          \right. \\
          \max_{l \in (i, k)} \\
          \left\{
            \begin{array}{l}
              (6) \;\;\; I[il \;\; n.] \quad \cdotp LN[lki \;\; .\overline{d}.dnn] \quad \cdotp \uwave{N}[kjl \;\; \overline{dd}\overline{d}.\overline{d}.] \\
              (7) \;\;\; RNX\cdotp [ilk \;\; nn.\overline{d}d\overline{d}] \quad I[lk \;\; ..] \quad \cdotp \uwave{LN}[kjl \;\; .\overline{d}..\overline{d}.] \\
            \end{array}
          \right. \\
        \end{array}
      \right. \\
    \end{array}
  \right. \\
\end{array} \\
& \begin{array}{l}
  B\cdotp [ijx \;\; nn\overline{dd}\overline{dd}] \leftarrow \max_{k \in (i, j)} \\
  \left\{
    \begin{array}{l}
      (8) \;\;\; \hat{L}\hat{N}\cdotp [ikx \;\; nn.\overline{d}\overline{dd}] \quad R\cdotp [kjx \;\; ...\overline{d}.\overline{d}] \\
      (9) \;\;\; \hat{L}\hat{N}\cdotp [ikx \;\; nn.\overline{d}\overline{dd}] \quad N\cdotp [kjx \;\; \overline{d}.d\overline{d}.\overline{d}] \\
      (10) \;\;\; \hat{L}\hat{N}\cdotp [ikx \;\; nn.\overline{d}\overline{dd}] \quad N\cdotp [kjx \;\; d.\overline{dd}.\overline{d}] \\
    \end{array}
  \right. \\
\end{array} \\
& \begin{array}{l}
\underdot{\hat{L}}[ijx \;\; \overline{dd}\overline{dd}\overline{dd}] \leftarrow \max_{k \in (i, j)} \\
  \left\{
    \begin{array}{l}
      (11) \;\;\; X[ikx \;\; .\overline{d}.dnn] \quad \cdotp \hat{L}\hat{N}[kji \;\; .\overline{d}.\overline{d}\overline{dd}] \\
      (12) \;\;\; X[ikx \;\; .\overline{d}.\overline{d}\overline{d}d] \quad \cdotp \hat{L}\hat{N}[kji \;\; .\overline{d}.\overline{d}\overline{dd}] \\
    \end{array}
  \right. \\
\end{array} \\
\end{flalign*}

\begin{flalign*}
\hspace{8mm}
& \begin{array}{l}
  \uwave{L}[ijx \;\; \overline{dd}\overline{dd}\overline{dd}] \leftarrow \max_{k \in (i, j)} \\
  \left\{
    \begin{array}{l}
      (13) \;\;\; LN[ikx \;\; .\overline{d}.d\overline{dd}] \quad \cdotp N[kji \;\; \overline{dd}\overline{dd}\overline{dd}] \\
      (14) \;\;\; LN[ikx \;\; .\overline{d}.\overline{d}\overline{d}d] \quad \cdotp N[kji \;\; \overline{dd}\overline{dd}\overline{dd}] \\
      (15) \;\;\; L[ikx \;\; .\overline{d}.d\overline{dd}] \quad I[kj \;\; ..] \\
      (16) \;\;\; L[ikx \;\; .\overline{d}.\overline{d}\overline{d}d] \quad I[kj \;\; ..] \\
      (17) \;\;\; N[ikx \;\; \overline{dd}dd\overline{dd}] \quad I[kj \;\; ..] \\
      (18) \;\;\; N[ikx \;\; \overline{dd}d\overline{d}\overline{d}d] \quad I[kj \;\; ..] \\
      (19) \;\;\; N[ikx \;\; d\overline{d}\overline{d}d\overline{dd}] \quad I[kj \;\; ..] \\
      (20) \;\;\; N[ikx \;\; d\overline{d}\overline{dd}\overline{d}d] \quad I[kj \;\; ..] \\
    \end{array}
  \right. \\
\end{array} \\
& \begin{array}{l}
  \uwave{N}[ijx \;\; \overline{dd}\overline{dd}\overline{dd}] \leftarrow \max_{k \in (i, j)} \\
  \left\{
    \begin{array}{l}
      (21) \;\;\; \cdotp N[ikx \;\; \overline{dd}\overline{d}d\overline{dd}] \quad I[kj \;\; ..] \\
      (22) \;\;\; \cdotp N[ikx \;\; \overline{dd}\overline{dd}\overline{d}d] \quad I[kj \;\; ..] \\
      (23) \;\;\; I[ik \;\; ..] \quad N\cdotp [kjx \;\; \overline{d}d\overline{dd}\overline{dd}] \\
      (24) \;\;\; I[ik \;\; ..] \quad N\cdotp [kjx \;\; \overline{dd}\overline{dd}d\overline{d}] \\
    \end{array}
  \right. \\
\end{array} \\
& \begin{array}{l}
  \underdot{N}[ijx \;\; \overline{dd}\overline{dd}\overline{dd}] \leftarrow \max_{k \in (i, j)} \\
  \left\{
    \begin{array}{l}
      (25) \;\;\; \cdotp X[ikx \;\; .\overline{d}.d\overline{dd}] \quad I[kj \;\; ..] \\
      (26) \;\;\; \cdotp X[ikx \;\; .\overline{d}.\overline{d}\overline{d}d] \quad I[kj \;\; ..] \\
      (27) \;\;\; I[ik \;\; ..] \quad X\cdotp [kjx \;\; .d.\overline{d}\overline{dd}] \\
      (28) \;\;\; I[ik \;\; ..] \quad X\cdotp [kjx \;\; .\overline{d}.\overline{d}d\overline{d}] \\
    \end{array}
  \right. \\
\end{array} \\
\end{flalign*}
\end{multicols}

\vspace{-8mm}
{\footnotesize
  $I[ij \;\; pn]$, 
$\cdotp B[ijx \;\; \overline{dd}nn\overline{dd}]$,
$\uwave{R}[ijx \;\; \overline{dd}\overline{dd}\overline{dd}]$,
}
and 
{\footnotesize $\underdot{R}[ijx \;\; \overline{dd}\overline{dd}\overline{dd}]$}
are symmetric with cases above.

\vspace{-1mm}
\caption{\label{alg:rules}
Dynamic program for Lock-Free, One-Endpoint Crossing, Directed, Acyclic graph parsing.
}
\end{algorithm*}

\paragraph{Deduction Rules}

To make the deduction rules manageable, we use templates to define some constraints explicitly, and then use code to generate the rules.
During rule generation, we automatically apply additional constraints to prevent rules that would leave a word in the middle of a span without a parent or that would form a cycle (proven possible below).
Algorithm~\ref{alg:rules} presents the explicit constraints.
Once expanded, these give rules that specify all properties for each item (general type, external vertex position relative to the item spans, connectivity of every pair of vertices in each item, etc).

The final item for $n$ vertices is an interval where the left end has a parent.
For parsing we assume there is a special root word at the end of the sentence.

\subsection{Properties}

\begin{definition} \label{def:1ec}
  A graph is \textbf{One-Endpoint Crossing} if, when drawn with vertices along the edge of a half-plane and edges drawn in the open half-plane above, for any edge $e$, all edges that cross $e$ share a vertex. Also let that vertex be \pencilPoint($e$).
\end{definition}

Aside from applying to graphs, this is the same as \textcite{ec}'s \oneEC tree definition.

\begin{definition} \label{def:prob-struct}
  A \textbf{\problemStruct} (shown in Fig.~\ref{fig:bad-structure}) is formed by a set of consecutive vertices in order from $0$ to $N$, where $N > 3$, with edges $\{ (0, N{-}1), (1, N)\} \cup \{(i, i{+}2) \forall i \in [0, N{-}2]\}$.
\end{definition}

\begin{definition} \label{def:lock-free}
  A graph is \textbf{Lock-Free} if it does not contain edges that form a \problemStruct.
\end{definition}

Note that in practise, most parse structures satisfy \oneEC, and the \problemStruct structure does not occur in the \ptb when using our head rules.

\begin{theorem} \label{theorem:complete}
  For the space of Lock-Free One-Endpoint Crossing graphs, the algorithm is sound, complete and gives unique decompositions.
\end{theorem}

Our proof is very similar in style and structure to \textcite{ec}.
The general approach is to consider the set of structures an item could represent, and divide them into cases based on properties of the internal structure.
Then we show how each case can be decomposed into items, taking care to ensure all the properties that defined the case are satisfied.
Uniqueness follows from having no ambiguity in how a given structure could be decomposed.
Completeness and soundness follow from the fact that our rules apply equally well in either direction, and so our top-down decomposition implies a bottom-up formation.
To give intuition for the proof, we show the derivation of one rule below.
The complete proof can be found in \textcite{Kummerfeld:PhD}.
We do not include it here due to lack of space.

We do provide the complete set of rule templates in Algorithm~\ref{alg:rules}, and in the proof of Lemma~\ref{lemma:B-case} we show that the case in which an item cannot be decomposed occurs if and only if the graph contains a \problemStruct.
To empirically check our rule generation code, we checked that our parser uniquely decomposes all \oneEC parses in the \ptb and is unable to decompose the rest.

Note that by using subsets of our rules, we can restrict the space of structures we generate, giving parsing algorithms for projective DAGs, projective trees \parencite{eisner:1996}, or \oneEC trees \parencite{ec}.
Also, versions of these spaces with undirected edges could be easily handled with the same approach.

\begingroup
\setlength\intextsep{0pt}
\begin{wrapfigure}[3]{r}{0pt}%
\begin{tikzpicture}%
  [every fit/.style={rectangle,rounded corners,draw,inner sep=2pt}]%
  \node (p) at (0, 0) {};
  \node (q) [right=2cm of p] {};
  \node (a) at ($(p)!0.15!(q)$) {};
  \node (s) at ($(p)!0.4!(q)$) {};
  \node (t) at ($(p)!0.65!(q)$) {};
  \node (b) at ($(p)!0.85!(q)$) {};

  \node (pText) [below=-3pt of p.center] {p\strut};
  \node (qText) [below=-3pt of q.center] {q\strut};
  \node (sText) [below=-3pt of s.center] {s\strut};
  \node (tText) [below=-3pt of t.center] {t\strut};

  \draw [thick] (p.center) -- (q.center);
  \draw [thick,out=40,in=140,color=black] (p.center) to node (ps) [midway] {} (s.center);
  \draw [thick,out=60,in=120,color=black] (a.center) to node (aa) [midway] {} (t.center);
  \draw [thick,out=60,in=120,color=black] (s.center) to node (sc) [midway] {} (b.center);
  \node () [fit=(p) (q) (ps) (aa) (pText) (qText) (sText) (tText)] {};
\end{tikzpicture}%
\end{wrapfigure}%
\tightparagraph{Derivation of rule (4) in Algorithm~\ref{alg:rules}}
This rule applies to intervals with the substructure shown, and with no parent in this item for $p$.
They have at least one $p$--$(pq)$ edge (otherwise rule 1 applies).
The longest $p$--$(pq)$ edge, $ps$, is crossed (otherwise rule 2 applies).
Let $C$ be the set of $(ps)$--$(sq)$ edges (note: these cross $ps$).
Either all of the edges in $C$ have a common endpoint $t \in (sq)$, or if $|C| = 1$ let $t$ be the endpoint in $(sq)$ (otherwise rule 6 or 7 applies).
Let $D$ be the set of $s$--$(tq)$ edges.
$|D| > 0$ (otherwise rule 3 or 5 applies).

We will break this into three items.
First, $(st)$--$(tq]$ edges would violate the \oneEC property and $(st)$--$[ps)$ edges do not exist by construction.
Therefore, the middle item is an Interval $[st]$, the left item is $[ps.t]$, and the right item is $[tq.s]$ (since $|C| > 0$ and $|D| > 0$).
The left item can be either an $N$ or $R$, but not an $L$ or $B$ because that would violate the \oneEC property for the $C$ edges.
The right item can be an $X$, $L$, or $N$, but not an $R$ or $B$ because that would violate the \oneEC property for the $D$ edges.
We will require edge $ps$ to be present in the first item, and not allow $pt$.
To avoid a spurious ambiguity, we also prevent the first or third items from having $st$ (which could otherwise occur in any of the three items).
Now we have broken down the original item into valid sub-items, and we have ensured that those sub-items contain all of the structure used to define the case in a unique way. \\
\endgroup

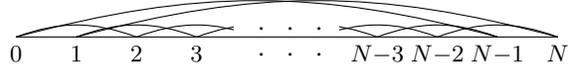
\begin{figure}
\centering
\begin{tikzpicture}
  \node (p) at (0, 0) {};
  \node (pText) [below=-2.0ex of p] {\footnotesize $0$\strut};
  \node (sp) at (0.8, 0) {};
  \node (spText) [below=-2.0ex of sp] {\footnotesize $1$\strut};
  \node (m1) at (1.6, 0) {};
  \node (m1Text) [below=-2.0ex of m1] {\footnotesize $2$\strut};
  \node (m2) at (2.4, 0) {};
  \node (m2Text) [below=-2.0ex of m2] {\footnotesize $3$\strut};
  \node (m3) at (3.2, 0) {};
  \node (m4) at (4, 0) {};
  \node (m5) at (4.8, 0) {};
  \node (m5Text) [below=-2.0ex of m5] {\footnotesize $N{-}3$\strut};
  \node (m6) at (5.6, 0) {};
  \node (m6Text) [below=-2.0ex of m6] {\footnotesize $N{-}2$\strut};
  \node (sq) at (6.4, 0) {};
  \node (sqText) [below=-2.0ex of sq] {\footnotesize $N{-}1$\strut};
  \node (q) at (7.2, 0) {};
  \node (qText) [below=-2.0ex of q] {\footnotesize $N$\strut};
  \draw [out=15,in=165] (p.north) to (sq.north);
  \draw [out=15,in=165] (sp.north) to (q.north);
  \draw [out=20,in=160] (p.north) to (m1.north);
  \draw [out=20,in=160] (sp.north) to (m2.north);
  \draw [out=20,in=160] (m1.north) to (m3.north);
  \draw [out=20,in=160] (m2.north) to (m4.north);
  \draw [out=20,in=160] (m3.north) to (m5.north);
  \draw [out=20,in=160] (m4.north) to (m6.north);
  \draw [out=20,in=160] (m5.north) to (sq.north);
  \draw [out=20,in=160] (m6.north) to (q.north);
  \node [fill=white] (dots) at (3.6,0.25) {\small $\;$ . $\;$ . $\;$ .$\;$ };
  \node [fill=white] (dots) at (3.6,-0.1) {\small $\;$ . $\;$ . $\;$ .$\;$ };
  \draw (p.north) -- (q.north);
\end{tikzpicture}
\vspace{-5mm}
\caption{\label{fig:bad-structure} Visualization of \problemStruct structures.
  Note, the use of $0$ to $N$ does not imply this must span the entire sentence, these numbers are just for convenience in the definition.
}
\end{figure}

Now we will further characterize the nature of the Lock-Free restriction to the space of graphs.

\begin{lemma} \label{lemma:chain-limit}
  No edge in a \problemStruct in a \oneEC graph is crossed by edges that are not part of it.
\end{lemma}
\begin{proof}
  First, note that:
  $\pencilPointMath((0, N{-}1)) = N$,
  $\pencilPointMath((1, N)) = 0$, and
  $\{\pencilPointMath((i, i{+}2)) = i{+}1 \; \forall i \in [0, N{-}2]\}$
  Also, call the set $\{(i, i{+}2) \forall i \in [0, N{-}2]\}$, the \emph{chain}.
  
  Consider an edge $e$ that crosses an edge $f$ in a \problemStruct.
  Let $e_{in}$ be the end of $e$ that is between the two ends of $f$, and $e_{out}$ be the other end.
  One of $e$'s endpoints is at $\pencilPointMath(f)$, and $\pencilPointMath(e)$ is an endpoint of $f$.
  There are three cases:

  (i) $f = (1, N)$.
  Here, $e_{out} = \pencilPointMath(f) = 0$, and $e_{in} \in (1, N)$.
  For all vertices $v \in (1, N)$ there is an edge $g$ in the chain such that $v$ is between the endpoints of $g$.
  Therefore, $e$ will cross such an edge $g$.
  To satisfy the \oneEC property, $g$ must share an endpoint with $f$, which means $g$ is either $(1, 3)$ or $(N{-}2, N)$.
  In the first case, the \oneEC property forces $e = (0, 2)$, and in the second $e = (0, N{-}1)$, both of which are part of the \problemStruct.
  
  (ii) $f = (0, N{-}1)$, symmetrical with (i).

  (iii) $f = (i, i{+}2)$, for some $i \in [0, N{-}2]$.
  Here, $e_{in} = \pencilPointMath(f) = i{+}1$.
  We can assume $e$ does not cross $(0, N{-}1)$ or $(1, N)$, as those cases are covered by (i).
  As in (i), $e$ must cross another edge in the chain, and that edge must share an endpoint with $f$.
  This forces $e$ to be either $(i{-}1, i{+}1)$ or $(i{+}1, i{+}3)$ (excluding one or both if they cross $(0, N{-}1)$ or $(1, N)$), which are both in the \problemStruct.
\end{proof}

Our rules define a unique way to decompose almost any item into a set of other items.
The exception is $B$, which in some cases can not be divided into two items (\myie has no valid binary division).

\begin{lemma} \label{lemma:B-case}
  A $B[ij.x]$ has no valid binary division if and only if the graph has a \problemStruct.
\end{lemma}
\begin{proof}
  Consider the $k$ and $l$ that give the longest $ik$ and $lj$ edges in a $B$ with no valid binary division (at least one edge of each type must exist by definition).
  No vertex in $(ik)$ or $(jl)$ is a valid split point, as they would all require one of the items to have two external vertices.

  Now, consider $p \in [kj]$.
  If there is no edge $l_1r_1$, where $i \le l_1 < p < r_1 \le j$, then $p$ would be a valid split point.
  Therefore, such an edge must exist.
  Consider $l_1$, either $l_1 \in (ik)$ or there is an edge $l_2c$, where $i \le l_2 < l_1 < c \le j$ (by the same logic as for $l_1r_1$).
  Similarly, either $r_1 \in (jl)$ or there is an edge $cr_2$ (it must be $c$ to satisfy \oneEC).
  We can also apply this logic to edges $l_2c$ and $cr_2$, giving edges $l_3l_1$ and $r_1r_3$.
  This pattern will terminate when it reaches $l_u \in (ik)$ and $r_v \in (jl)$ with edges $l_ul_{u-2}$ and $r_{v-2}r_v$.
  Note that $k=l_{u-1}$ and $l=r_{v-1}$, to satisfy \oneEC.

  Since it is a $B$, there must be at least two $x$--$(ij)$ edges.
  To satisfy \oneEC, these end at $l_{u-1}$ and $r_{v-1}$.

  Let $x$ be to the right (the left is symmetrical), and call $i=0$, $j=N{-}1$, and $x=N$.
  Comparing with the \problemStruct definition, we have all the edges except one: $0$ to $N{-}1$.
  However, that edge must be present in the overall graph, as all $B$ items start with an $ij$ edge (see rules 3 and 5 in Algorithm~\ref{alg:rules}).
  Therefore, if there is no valid split point for a $B$, the overall graph must contain a \problemStruct.

  Now, for a graph that contains a \problemStruct, consider the items that contain the \problemStruct.
  Grouping them by their span $[ij]$, there are five valid options: $[0, N{-}1]$, $[1, N]$, $[0, N]$, ($i \leq 0 \land j > N$), and ($i < 0 \land j \geq N$).
  Items of the last three types would be divided by our rules into smaller items, one of which contains the whole \problemStruct.
  The first two are $B$s of the type discussed above.
\end{proof}

Now we will prove that our code to generate rules from the templates can guarantee a DAG is formed.

\begin{lemma} \label{lemma:reach}
  For any item $H$, $\forall v \in \vertexCovMath(H)$ \linebreak
  $\exists u \in \vertexVisMath(H): v$ is reachable from $u$.
\end{lemma}

\begin{proof}
  This is true for initial items because $\vertexCovMath(H) = \emptyset$.
  To apply induction, consider adding edges and combing items.
  The lemma clearly remains true when adding an edge.
  Consider combining items $E$, $F$, $G$ to form $H[ij.x]$, and assume the lemma is true for $E$, $F$, and $G$ (the binary case is similar).
  Since all vertices are reachable from $\vertexVisMath(E,F,G)$, we only need to ensure that $\forall v \in \vertexVisMath(E,F,G)\; \exists u \in \vertexVisMath(H): v$ is reachable from $u$.
  The connectivity between all pairs $\{(u, v) \mid u \in \vertexVisMath(H), v \in \vertexVisMath(E,F,G)\}$ can be inferred from the item definitions, and so this requirement can be enforced in rule generation.
\end{proof}

\begin{lemma} \label{lemma:dag}
  The final item is a directed acyclic graph.
\end{lemma}
\begin{proof}
  First, consider acyclicity.
  Initial items do not contain any edges and so cannot contain a cycle.
  For induction, there are two cases:

  (i) Adding an Edge $\vec{pq}$ to an item $H$:
  Assume that $H$ does not contain any cycles.
  $\vec{pq}$ will create a cycle if and only if $p$ is reachable from $q$.
  By construction, $p$ and $q \in \vertexVisMath(H)$, and so the item definition contains whether $p$ is reachable from $q$.

  (ii) Combining Items:
  Assume that in isolation, none of the items being combined contain cycles.
  Therefore, a cycle in the combined item must be composed of paths in multiple items.
  A path in one item can only continue in another item by passing through a \vertexVis vertex.
  Therefore, a cycle would have to be formed by a set of paths between \vertexVis vertices.
  But the connectivity of every pair of \vertexVis vertices is specified in the item definitions.

  In both cases, rules that create a cycle can be excluded during rule generation.

  By induction, the items constructed by our algorithm do not contain cycles.
  Together with Lemma~\ref{lemma:reach} and the final item definition, this means the final structure is an acyclic graph with all vertices reachable from vertex $n$.
\end{proof}

Next, we will show two properties that give intuition for the algorithm.
Specifically, we will prove which rules add edges that are crossed in the final derivation.

\begin{lemma} \label{lemma:no-int-cross}
  An edge $ij$ added to $I[ij]$ is not crossed.
\end{lemma}
\begin{proof}
  First, we will show three properties of any pair of items in a derivation (using $[ij.x]$ and $[kl.y]$).
  
  (1) \textit{It is impossible for either $i < k < j < l$ or $k < i < l < j$}, \myie items cannot have partially overlapping spans.
  As a base case, the final item is an interval spanning all vertices, and so no other item can partially overlap with it.
  Now assume it is true for an item $H$ and consider the rules in reverse, breaking $H$ up.
  By construction, each rule divides $H$ into items with spans that are adjacent, overlapping only at their \vertexVis vertices.
  Also, since the new items are nested within $H$, they do not overlap with any items $H$ did not overlap with.
  By induction, no pair of items have partially overlapping spans.
  
  (2) \textit{For items with nested spans ($i \le k < l \le j$), $y \in [ij] \cup \{x\}$}.
  Following the argument for the previous case, the $[ij.x]$ item must be decomposed into a set of items that includes $[kl.y]$.
  Now, consider how those items are combined.
  The rules that start with an item with an external vertex produce an item that either has the same external vertex, or with the external vertex inside the span of the new item.
  Therefore, $y$ must either be equal to $x$ or inside $[ij]$.

  (3) \textit{For items without nested spans, $x \notin (kl)$}.
  Assume $x \in (kl)$ for two items without nested spans.
  None of the rules combine such a pair of items, or allow one to be extended so that the other is nested within it.
  But all items are eventually combined to complete the derivation.
  By contradiction, $x \notin (kl)$.

  Together, these mean that given an interval $H$ with span $[ij]$, and another item $G$, either $\forall v \in \vertexVisMath(G), v \in [ij]$ or $\forall v \in \vertexVisMath(G), v \notin (ij)$.
  Since edges are only created between \vertexVis vertices, no edge can cross edge $ij$.
\end{proof}

\begin{lemma} \label{lemma:ext-crossing}
  All edges aside from those considered in Lemma~\ref{lemma:no-int-cross} are crossed.
\end{lemma}
\begin{proof}
  First, consider an edge $ij$ added to an item $[ij.x]$ of type B, L, R, or N.
  This edge is crossed by all $x$--$(ij)$ edges, and in these items $|x$--$(ij)| \ge 1$ by definition.
  Note, by the same argument as Lemma~\ref{lemma:no-int-cross}, the edge is not crossed later in the derivation.

  Second, consider adding $e \in \{xi, xj\}$, to $H$, an item with $[ij]$ or $[ij.x]$, forming an item $G[ij.x]$.
  Note, $e$ does not cross any edges in $H$.
  Let $E(F[kl.y])$ be the set of $y$--$[kl]$ edges in some item $F$.
  Note that $e \in E(G)$.
  We will show how this set of edges is affected by the rules and what that implies for $e$.
  Consider each input item $A[kl.y]$ in each rule, with output item $C$.
  Every item $A$ falls into one of four categories:
  (1) $\forall f \in E(A), f$ is crossed by an edge in another of the rule's input items,
  (2) $E(A) \subseteq E(C)$,
  (3) $A \land kl \mapsto C$ and there are no $ky$ or $ly$ edges in $A$,
  (4) $A$ contains edge $kl$ and there are no $ky$ or $ly$ edges in $A$.

  Cases 2-4 are straightforward to identify.
  For an example of the first case, consider the rightmost item in rule 4.
  The relevant edges are $k$--$(lj]$ (by construction, $kl$ is not present).
  Since the leftmost item is either an R or N, $|l$--$(ik)| \ge 1$.
  Since $i < k < l < j$, all $k$--$(lj]$ edges will cross all $l$--$[ik)$ edges.
  Therefore applying this rule will cross all $k$--$(lj]$ edges in the rightmost item.

  Initially, $e$ is not crossed and $e \in E(G)$.
  For each rule application, edges in $E(A)$ are either crossed (1 and 3), remain in the set $E(C)$ (2), or must already be crossed (4).
  Since the final item is an interval and $E($Interval$) = \emptyset$, there must be a subsequent rule that is not in case 2.
  Therefore $e$ will be crossed.
\end{proof}

\subsection{Comparison with \textcite{ec}} \label{sec:ec-comparison}

Our algorithm is based on \textcite{ec}, which had the crucial idea of one-endpoint crossing and a complete decomposition of the tree case.
Our changes and extensions provide several benefits:

\tightparagraph{Extension to graphs}
By extending to support multiple parents while preventing cycles, we substantially expand the space of generatable structures.

\tightparagraph{Uniqueness}
By avoiding derivational ambiguity we reduce the search space and enable efficient summing as well as maxing.
Most of the cases in which ambiguity arises in \textcite{ec}'s algorithm are due to symmetry that is not explicitly broken.
For example, the rule we worked through in the previous section defined $t \in (sq)$ when $|C| = 1$.
Picking $t \in (ps)$ would also lead to a valid set of rules, but allowing either creates a spurious ambiguity.
This ambiguity is resolved by tracking whether there is only one edge to the external vertex or more than one, and requiring more than one in rules 6 and 7.
Other changes include ensuring equivalent structures cannot be represented by multiple item types and enforcing a unique split point in $B$ items.

\tightparagraph{More concise algorithm definition}
By separating edge creation from item merging, and defining our rules via a combination of templates and code, we are able to define our algorithm more concisely.

\subsection{Algorithm Extensions}

\subsubsection{Edge Labels and Word Labels}\label{sec:labels}
Edge labels can be added by calculating either the sum or max over edge types when adding each edge.
Word labels (\myeg POS Tags) must be added to the state, specifying a label for each visible word ($p$, $q$ and $o$).
This state expansion is necessary to ensure agreement when combining items.

\subsubsection{Ensuring a Structural Tree is Present}
Our algorithm constrains the space of graph structures, but we also want to ensure that our parse contains a projective tree of non-trace edges.

To ensure every word gets one and only one structural parent, we add booleans to the state, indicating whether $p$, $q$ and $o$ have structural parents.
When adding edges, a structural edge cannot be added if a word already has a structural parent.
When combining items, no word can receive more than one structural parent, and words that will end up in the middle of the span must have exactly one.
Together, these constraints ensure we have a tree.

To ensure the tree is projective, we need to prevent structural edges from crossing.
Crossing edges are introduced in two ways, and in both we can avoid structural edges crossing by tracking whether there are structural $o$--$[pq]$ edges.
Such edges are present if a rule adds a structural $op$ or $oq$ edge, or if a rule combines an item with structural $o$--$[pq]$ edges and $o$ will still be external in the item formed by the rule.

For adding edges, every time we add a $pq$ edge in the $N$, $L$, $R$ and $B$ items we create a crossing with all $o$--$(pq)$ edges.
We do not create a crossing with $oq$ or $op$, but our ordering of edge creation means these are not present when we add a $pq$ edge, so tracking structural $o$--$[pq]$ edges gives us the information we need to prevent two structural edges crossing.

For combining items, in Lemma~\ref{lemma:ext-crossing} we showed that during combinations, $o$--$[pq]$ edges in each pair of items will cross.
As a result, knowing whether any $o$--$[pq]$ edge is structural is sufficient to determine whether two structural edges will cross.

\subsection{Complexity}

Consider a sentence with $n$ tokens, and let $E$ and $S$ be the number of edge types and word labels in our grammar respectively.

\tightparagraph{Parses without word or edge labels}
Rules have up to four positions, leading to complexity of $O(n^4)$.
Note, there is also an important constant--once our templates are expanded, there are 49,292 rules.

\tightparagraph{With edge labels}
When using a first-order model, edge labels only impact the rules for edge creation, leading to a complexity of $O(n^4 + E n^2)$.

\tightparagraph{With word labels}
Since we need to track word labels in the state, we need to adjust every $n$ by a factor of $S$, leading to $O(S^4 n^4 + E S^2 n^2)$.

%
%

\section{Parse Representation} \label{sec:representation}

Our algorithm relies on the assumption that we can process the dependents to the left and right of a word independently and then combine the two halves.
This means we need lexicalized structures, which the \ptb does not provide.
We define a new representation in which each non-terminal symbol is associated with a specific word (the head).
Unlike dependency parsing, we retain all the information required to reconstruct the constituency parse.

Our approach is related to \textcite{cck} and \textcite{hayashi-nagata:2016}, with three key differences:
(1) we encode non-terminals explicitly, rather than implicitly through adjunction operations, which can cause ambiguity,
(2) we add representations of null elements and co-indexation,
(3) we modify head rules to avoid problematic structures.

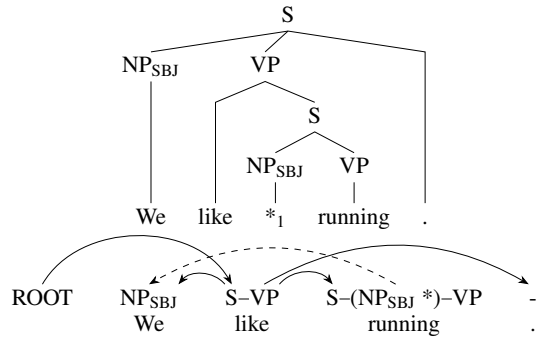
\begin{figure}
\centering
\scalebox{0.8}{
\begin{tikzpicture}
[ every node/.style={
    node distance=2ex,
    inner sep=0pt
  },
  structural/.style={
    ->,
    >=Stealth,
    thin,
  },
  trace/.style={
    ->,
    >=Stealth,
    thin,
    dashed
  },
]
  \node (w0) at (0, 0) {\strut We};
  \node (w1) [right=3ex of w0.east] {\strut like};
  \node (w2) [right=3ex of w1.east] {\strut *\textsubscript{1}};
  \node (w3) [right=3ex of w2.east] {\strut running};
  \node (w4) [right=3ex of w3.east] {\strut .};

  \node (nt00) [above=of w0.north] {\strut};
  \node (nt10) [above=of w1.north] {\strut};
  \node (nt20) [above=of w2.north] {\strut NP\textsubscript{SBJ}};
  \node (nt30) [above=of w3.north] {\strut VP};
  \node (nt40) [above=of w4.north] {\strut};

  \node (nt01) [above=of nt00.north] {\strut};
  \node (nt11) [above=of nt10.north] {\strut};
  \path (nt20.north) -- node[above=2ex] (nt21-31) {\strut S} (nt30.north);
  \node (nt41) [above=of nt40.north] {\strut};

  \node (nt02) [above=of nt01.north] {\strut NP\textsubscript{SBJ}};
  \path (nt11.north) -- node[above=2ex] (nt12-32) {\strut VP} (nt21-31.north);
  \node (nt42) [above=of nt41.north] {\strut};

  \path (nt02.north) -- node[above=2ex] (nt03-43) {\strut S} (nt42.north);

  \draw (w0.north) -- (nt02.south);
  \draw (w1.north) -- (nt11.north) -- (nt12-32.south) -- (nt21-31.north);
  \draw (w2.north) -- (nt20.south);
  \draw (nt20.north) -- (nt21-31.south) -- (nt30.north);
  \draw (w3.north) -- (nt30.south);
  \draw (w4.north) -- (nt42.north) -- (nt03-43.south) -- (nt02.north);
  \draw (nt12-32.north) -- (nt03-43.south);
\end{tikzpicture}
}

\vspace{-2mm}
\scalebox{0.8}{
\centering
\begin{tikzpicture}
[ every node/.style={
    node distance=2ex,
    inner sep=0pt
  },
  structural/.style={
    ->,
    >=Stealth,
    thin,
  },
  trace/.style={
    ->,
    >=Stealth,
    thin,
    dashed
  },
]

  \node (sntR0) at (0, 1.8) {\strut ROOT};

  \node (snt00) [right=2em of sntR0.east] {\strut NP\textsubscript{SBJ}};
  \node (sw0) [below=-0.3ex of snt00] {\strut We};

  \node (snt10) [right=2em of snt00.east] {\strut S};
  \node (snt11) [right=1ex of snt10.east] {\strut VP};
  \draw (snt10.east) -- (snt11.west);
  \node (sw1a) at ($(snt10.west)!0.5!(snt11.east)$) {\strut};
  \node (sw1) [below=-0.3ex of sw1a] {\strut like};

  \node (snt20) [right=2em of snt11.east] {\strut S};
  \node (snt21) [right=1ex of snt20.east] {\strut (NP\textsubscript{SBJ} *)};
  \draw (snt20.east) -- (snt21.west);
  \node (snt22) [right=1ex of snt21.east] {\strut VP};
  \draw (snt21.east) -- (snt22.west);
  \node (sw2a) at ($(snt20.west)!0.5!(snt22.east)$) {\strut};
  \node (sw2) [below=-0.3ex of sw2a] {\strut running};

  \node (snt40) [right=2em of snt22.east] {\strut -};
  \node (sw4) [below=-0.3ex of snt40] {\strut .};

  \draw [structural,out=60,in=120] (sntR0.north) to (snt10.north);
  \draw [structural,out=130,in=50] (snt10.north west) to (snt00.north east);
  \draw [structural,out=50,in=130] (snt11.north east) to (snt20.north);
  \draw [structural,out=30,in=150] (snt11.north) to (snt40.north west);
  \draw [trace,out=150,in=30] (snt21.north) to (snt00.north);
\end{tikzpicture}}%
\vspace{-4mm}
  \caption{\label{fig:repr2}
    Parse representations for graph structures, \ptb (top) and ours (bottom).
  }
\end{figure}

Figure~\ref{fig:repr2} shows a comparison of the \ptb representation and ours.
We add lexicalization, assigning each non-terminal to a word.
The only other changes are visual notation, with non-terminals moved to be directly above the words to more clearly show the distinction between \emph{spines} and \emph{edges}.

\tightparagraph{Spines}
Each word is assigned a spine, shown immediately above the word.
A spine is the ordered set of non-terminals that the word is the head of, \myeg S-VP for \emph{like}.
If a symbol occurs more than once in a spine, we use indices to distinguish instances.

\tightparagraph{Edges}
An edge is a link between two words, with a label indicating the symbols it links in the child and parent spines.
In our figures, edge labels are indicated by where edges start and end.

\tightparagraph{Null Elements}
We include each null element in the spine of its parent, unlike \textcite{hayashi-nagata:2016}, who effectively treated null elements as words, assigning them independent spines.
We also considered encoding null elements entirely on edges but found this led to poorer performance.

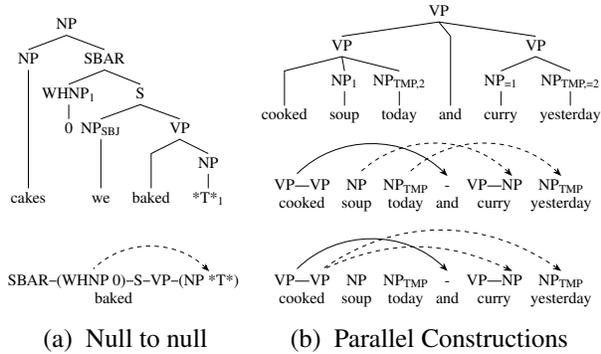
\begin{figure}
\begin{subfigure}[b]{0.21\textwidth}
  \centering
  \scalebox{0.55}{\begin{tikzpicture}
[ every node/.style={
    node distance=2ex,
    inner sep=0pt
  },
  structural/.style={
    <-,
    shorten >=2pt,
    >=Stealth,
    thin
  },
  trace/.style={
    <-,
    shorten >=2pt,
    >=Stealth,
    thin,
    dashed
  },
]

  \node (w0) at (0, 0) {\strut cakes};
  \node (w1) [right=3ex of w0.east] {\strut};
  \node (w2) [right=3ex of w1.east] {\strut we};
  \node (w3) [right=3ex of w2.east] {\strut baked};
  \node (w4) [right=3ex of w3.east] {\strut *T*\textsubscript{1}};

  \node (nt00) [above=of w0.north] {\strut};
  \node (nt10) [above=of w1.north] {\strut};
  \node (nt20) [above=of w2.north] {\strut};
  \node (nt30) [above=of w3.north] {\strut};
  \node (nt40) [above=of w4.north] {\strut NP};
  \draw (w4.north) -- (nt40.south);

  \node (nt01) [above=of nt00.north] {\strut};
  \node (nt11) [above=of nt10.north] {\strut 0};
  \node (nt21) [above=of nt20.north] {\strut NP\textsubscript{SBJ}};
  \draw (w2.north) -- (nt21.south);
  \path (nt30.north) -- node[above=2ex] (nt31-41) {\strut VP} (nt40.north);
  \draw (w3.north) -- (nt30.north) -- (nt31-41.south) -- (nt40.north);

  \node (nt02) [above=of nt01.north] {\strut};
  \node (nt12) [above=of nt11.north] {\strut WHNP\textsubscript{1}};
  \draw (nt11.north) -- (nt12.south);
  \path (nt21.north) -- node[above=2ex] (nt22-42) {\strut S} (nt31-41.north);
  \draw (nt21.north) -- (nt22-42.south) -- (nt31-41.north);

  \node (nt03) [above=of nt02.north] {\strut NP};
  \draw (w0.north) -- (nt03.south);
  \path (nt12.north) -- node[above=2ex] (nt13-43) {\strut SBAR} (nt22-42.north);
  \draw (nt12.north) -- (nt13-43.south) -- (nt22-42.north);

  \path (nt03.north) -- node[above=2ex] (nt04-44) {\strut NP} (nt13-43.north);
  \draw (nt03.north) -- (nt04-44.south) -- (nt13-43.north);

  \node (snt20) at (0.0, -2) {\strut SBAR};
  \node (snt21) [right=1ex of snt20.east] {\strut (WHNP 0)};
  \node (snt22) [right=1ex of snt21.east] {\strut S};
  \node (snt23) [right=1ex of snt22.east] {\strut VP};
  \node (snt24) [right=1ex of snt23.east] {\strut (NP *T*)};
  \node (sw2a) at ($(snt20.east)!0.5!(snt24.west)$) {\strut};
  \node (sw2) [below=-0.3ex of sw2a] {\strut baked};
  \draw (snt20.east) -- (snt21.west);
  \draw (snt21.east) -- (snt22.west);
  \draw (snt22.east) -- (snt23.west);
  \draw (snt23.east) -- (snt24.west);

  \draw [trace,out=135,in=45] (snt24.north) to (snt21.north);
\end{tikzpicture}}
  \caption{\label{fig:null-null}
    Null to null
  }
\end{subfigure}
\hfill
\begin{subfigure}[b]{0.27\textwidth}
  \centering
  \scalebox{0.55}{\begin{tikzpicture}
[ every node/.style={
    node distance=2ex,
    inner sep=0pt
  },
  structural/.style={
    <-,
    shorten >=2pt,
    >=Stealth,
    thin
  },
  trace/.style={
    <-,
    shorten >=2pt,
    >=Stealth,
    thin,
    dashed
  },
]

  \node (w0) at (0, 0) {\strut cooked};
  \node (w1) [right=3ex of w0.east] {\strut soup};
  \node (w2) [right=3ex of w1.east] {\strut today};
  \node (w3) [right=3ex of w2.east] {\strut and};
  \node (w4) [right=3ex of w3.east] {\strut curry};
  \node (w5) [right=3ex of w4.east] {\strut yesterday};

  \node (nt00) [above=of w0.north] {\strut};
  \node (nt10) [above=of w1.north] {\strut NP\textsubscript{1}};
  \node (nt20) [above=of w2.north] {\strut NP\textsubscript{TMP,2}};
  \node (nt30) [above=of w3.north] {\strut};
  \node (nt40) [above=of w4.north] {\strut NP\textsubscript{=1}};
  \node (nt50) [above=of w5.north] {\strut NP\textsubscript{TMP,=2}};

  \path (nt00.north) -- node[above=2ex] (nt01-21) {\strut VP} (nt20.north);
  \node (nt31) [above=of nt30.north] {\strut};
  \path (nt40.north) -- node[above=2ex] (nt41-51) {\strut VP} (nt50.north);

  \path (nt01-21.north) -- node[above=2ex] (nt02-52) {\strut VP} (nt41-51.north);

  \draw (w0.north) -- (nt00.north) -- (nt01-21.south) -- (nt20.north);
  \draw (w1.north) -- (nt10.south);
  \draw (nt10.north) -- (nt01-21.south);
  \draw (w2.north) -- (nt20.south);
  \draw (w3.north) -- (nt31.north) -- (nt02-52.south);
  \draw (w4.north) -- (nt40.south);
  \draw (w5.north) -- (nt50.south);
  \draw (nt01-21.north) -- (nt02-52.south) -- (nt41-51.north);
  \draw (nt40.north) -- (nt41-51.south) -- (nt50.north);

  \node (snt00) at (0.0, -1.7) {\strut VP};
  \node (snt01) [right=2ex of snt00.east] {\strut VP};
  \draw (snt00.east) -- (snt01.west);
  \node (snt10) [right=1em of snt01.east] {\strut NP};
  \node (snt20) [right=1em of snt10.east] {\strut NP\textsubscript{TMP}};
  \node (snt30) [right=1em of snt20.east] {\strut -};
  \node (snt40) [right=1em of snt30.east] {\strut VP};
  \node (snt41) [right=2ex of snt40.east] {\strut NP};
  \draw (snt40.east) -- (snt41.west);
  \node (snt50) [right=1em of snt41.east] {\strut NP\textsubscript{TMP}};
  \node (sw0a) at ($(snt00.west)!0.5!(snt01.east)$) {\strut};
  \node (sw0) [below=-0.3ex of sw0a.south] {\strut cooked};
  \node (sw1) [below=-0.3ex of snt10.south] {\strut soup};
  \node (sw2) [below=-0.3ex of snt20.south] {\strut today};
  \node (sw3) [below=-0.3ex of snt30.south] {\strut and};
  \node (sw4a) at ($(snt40.west)!0.5!(snt41.east)$) {\strut};
  \node (sw4) [below=-0.3ex of sw4a.south] {\strut curry};
  \node (sw5) [below=-0.3ex of snt50.south] {\strut yesterday};

  \draw [name path=s2,structural,semithick] (snt30.north) to [out=135,in=45] (snt00.north east);
  \draw [name path=t0,trace,semithick] (snt41.north) to [out=135,in=45] (snt10.north);
  \draw [name path=t1,trace,semithick] (snt50.north) to [out=135,in=45] (snt20.north);

  \node (snt200) at (0.0, -4) {\strut VP};
  \node (snt201) [right=2ex of snt200.east] {\strut VP};
  \draw (snt200.east) -- (snt201.west);
  \node (snt210) [right=1em of snt201.east] {\strut NP};
  \node (snt220) [right=1em of snt210.east] {\strut NP\textsubscript{TMP}};
  \node (snt230) [right=1em of snt220.east] {\strut -};
  \node (snt240) [right=1em of snt230.east] {\strut VP};
  \node (snt241) [right=2ex of snt240.east] {\strut NP};
  \draw (snt240.east) -- (snt241.west);
  \node (snt250) [right=1em of snt241.east] {\strut NP\textsubscript{TMP}};
  \node (sw20a) at ($(snt200.west)!0.5!(snt201.east)$) {\strut};
  \node (sw20) [below=-0.3ex of sw20a.south] {\strut cooked};
  \node (sw21) [below=-0.3ex of snt210.south] {\strut soup};
  \node (sw22) [below=-0.3ex of snt220.south] {\strut today};
  \node (sw23) [below=-0.3ex of snt230.south] {\strut and};
  \node (sw24a) at ($(snt240.west)!0.5!(snt241.east)$) {\strut};
  \node (sw24) [below=-0.3ex of sw24a.south] {\strut curry};
  \node (sw25) [below=-0.3ex of snt250.south] {\strut yesterday};

  \draw [name path=s2,structural,semithick] (snt230.north) to [out=135,in=45] (snt200.north east);
  \draw [name path=t0,trace,semithick] (snt241.north) to [out=155,in=25] (snt201.north);
  \draw [name path=t1,trace,semithick] (snt250.north) to [out=145,in=35] (snt201.north);
\end{tikzpicture}%
}
  \caption{\label{fig:gapping}
    Parallel Constructions
  }
\end{subfigure}
\vspace{-8mm}
\caption{
  Examples of syntactic phenomena. 
  Only relevant edges and spines are shown.
}
\end{figure}

\tightparagraph{Co-indexation}
The treebank represents movement with index pairs on null elements and non-terminals, \myeg *\textsubscript{1} and NP\textsubscript{1} in Figure~\ref{fig:repr2}.
We represent co-indexation with edges, one per reference, going from the null element to the non-terminal.
There are three special cases of co-indexation:

\noindent
\textbf{(1)}
It is possible for trace edges to have the same start and end points as a non-trace edge.
We restrict this case to allow at most one trace edge.
This decreases edge coverage in the training set by 0.006\%.

\noindent
\textbf{(2)}
In some cases the reference non-terminal only spans a null element, \myeg the WHNP in Figure~\ref{fig:null-null}.
For these we use a reversed edge to avoid creating a cycle.
Figure~\ref{fig:null-null} shows a situation where the trace edge links two positions in the same spine, which we assign with the spine during parsing.

\noindent
\textbf{(3)}
For parallel constructions the treebank co-indexes arguments that fulfill the same roles (Fig.~\ref{fig:gapping}).
These are distinct from the previous cases because neither index is on a null element.
We considered two options: add edges from the repetition to the referent (middle), or add edges from the repetition to the parent of the first occurrence (bottom).
Option two produces fewer non-\oneEC structures and explicitly represents all predicates, but only implicitly captures the original structure.

\subsection{Avoiding Adjunction Ambiguity}

Prior work on parsing with spines has used r-adjunction to add additional non-terminals to spines.
This introduces ambiguity, because edges modifying the same spine from different sides may not have a unique order of application.
We resolve this issue by using more articulated spines with the complete set of non-terminals.
We found that $0.045\%$ of spine instances in the development set are not observed in training, though in $70\%$ of those cases an equivalent spine sans null elements is observed in training.

\subsection{Head Rules} \label{sec:rep-head}

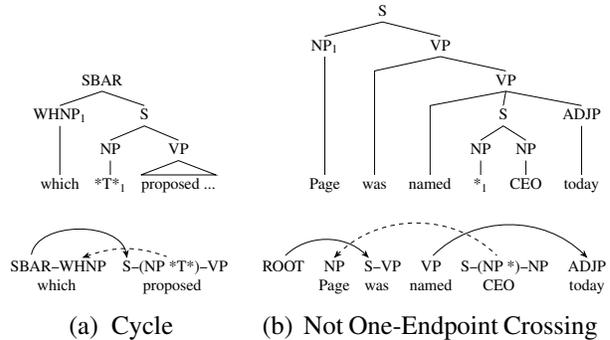
\begin{figure}
\begin{subfigure}[b]{0.19\textwidth}
  \centering
  \scalebox{0.55}{\begin{tikzpicture}
[ every node/.style={
    node distance=2ex,
    inner sep=0pt
  },
  structural/.style={
    <-,
    shorten >=2pt,
    >=Stealth,
    thin
  },
  trace/.style={
    <-,
    shorten >=2pt,
    >=Stealth,
    thin,
    dashed
  },
]

  \node (w0) at (0.7, 0) {\strut which};
  \node (w1) [right=2ex of w0.east] {\strut *T*\textsubscript{1}};
  \node (w2) [right=2ex of w1.east] {\strut proposed ...};

  \node (nt01) [above=of w0.north] {\strut};
  \node (nt11) [above=of w1.north] {\strut NP};
  \draw (w1.north) -- (nt11.south);
  \node (nt21) [above=of w2.north] {\strut VP};
  \draw (w2.north west) -- (nt21.south) -- (w2.north east) -- (w2.north west);

  \node (nt02) [above=of nt01.north] {\strut WHNP\textsubscript{1}};
  \draw (w0.north) -- (nt02.south);
  \path (nt11.north) -- node[above=2ex] (nt12-22) {\strut S} (nt21.north);
  \draw (nt11.north) -- (nt12-22.south) -- (nt21.north);

  \path (nt02.north) -- node[above=2ex] (nt03-23) {\strut SBAR} (nt12-22.north);
  \draw (nt02.north) -- (nt03-23.south) -- (nt12-22.north);

  \node (snt00) at (0, -2) {\strut SBAR};
  \node (snt01) [right=1ex of snt00.east] {\strut WHNP};
  \draw (snt00.east) -- (snt01.west);
  \node (sw0a) at ($(snt00)!0.5!(snt01)$) {\strut};
  \node (sw0) [below=-0.3ex of sw0a.south] {\strut which};
  \node (snt10) [right=1em of snt01.east] {\strut S};
  \node (snt11) [right=1ex of snt10.east] {\strut (NP *T*)};
  \draw (snt10.east) -- (snt11.west);
  \node (snt12) [right=1ex of snt11.east] {\strut VP};
  \draw (snt11.east) -- (snt12.west);
  \node (sw1a) at ($(snt10)!0.5!(snt12)$) {\strut};
  \node (sw1) [below=-0.3ex of sw1a.south] {\strut proposed};

  \draw [trace,out=20,in=160] (snt01.north) to (snt11.north);
  \draw [structural,out=100,in=80] (snt10.north) to (snt00.north);
\end{tikzpicture}}
  \caption{\label{fig:cycle}
    Cycle
  }
\end{subfigure}
\hfill
\begin{subfigure}[b]{0.27\textwidth}
  \centering
  \scalebox{0.55}{\begin{tikzpicture}
[ every node/.style={
    node distance=2ex,
    inner sep=0pt
  },
  structural/.style={
    <-,
    shorten >=2pt,
    >=Stealth,
    thin
  },
  trace/.style={
    <-,
    shorten >=2pt,
    >=Stealth,
    thin,
    dashed
  },
]

  \node (w0) at (1, 0) {\strut Page};
  \node (w1) [right=3ex of w0.east] {\strut was};
  \node (w2) [right=3ex of w1.east] {\strut named};
  \node (w3) [right=3ex of w2.east] {\strut *\textsubscript{1}};
  \node (w4) [right=3ex of w3.east] {\strut CEO};
  \node (w5) [right=3ex of w4.east] {\strut today};

  \node (nt01) [above=of w0.north] {\strut};
  \node (nt11) [above=of w1.north] {\strut};
  \node (nt21) [above=of w2.north] {\strut};
  \node (nt31) [above=of w3.north] {\strut NP};
  \node (nt41) [above=of w4.north] {\strut NP};
  \node (nt51) [above=of w5.north] {\strut};

  \node (nt02) [above=of nt01.north] {\strut};
  \node (nt12) [above=of nt11.north] {\strut};
  \node (nt22) [above=of nt21.north] {\strut};
  \path (nt31.north) -- node[above=2ex] (nt32-42) {\strut S} (nt41.north);
  \node (nt52) [above=of nt51.north] {\strut ADJP};

  \node (nt03) [above=of nt02.north] {\strut};
  \node (nt13) [above=of nt12.north] {\strut};
  \path (nt22.north) -- node[above=2ex] (nt23-53) {\strut VP} (nt52.north);

  \node (nt04) [above=of nt03.north] {\strut NP\textsubscript{1}};
  \path (nt13.north) -- node[above=2ex] (nt14-54) {\strut VP} (nt23-53.north);

  \path (nt04.north) -- node[above=2ex] (nt05-55) {\strut S} (nt14-54.north);

  \draw (w0.north) -- (nt04.south);
  \draw (nt04.north) -- (nt05-55.south) -- (nt14-54.north);
  \draw (w1.north) -- (nt13.north) -- (nt14-54.south) -- (nt23-53.north);
  \draw (w2.north) -- (nt22.north) -- (nt23-53.south) -- (nt52.north);
  \draw (nt32-42.north) -- (nt23-53.south);
  \draw (nt31.north) -- (nt32-42.south) -- (nt41.north);
  \draw (w3.north) -- (nt31.south);
  \draw (w4.north) -- (nt41.south);
  \draw (w5.north) -- (nt52.south);


  \node (sntR0) at (0, -2) {\strut ROOT};
  \node (snt00) [right=1.2em of sntR0.east] {\strut NP};
  \node (snt10) [right=1.2em of snt00.east] {\strut S};
  \node (snt11) [right=1ex of snt10.east] {\strut VP};
  \node (snt20) [right=1.2em of snt11.east] {\strut VP};
  \node (snt30) [right=1.2em of snt20.east] {\strut S};
  \node (snt31) [right=1ex of snt30.east] {\strut (NP *)};
  \node (snt32) [right=1ex of snt31.east] {\strut NP};
  \node (snt40) [right=1.2em of snt32.east] {\strut ADJP};
  \node (sw0) [below=-0.3ex of snt00.south] {\strut Page};
  \node (sw1a) at ($(snt10.east)!0.5!(snt11.west)$) {\strut};
  \draw (snt10.east) -- (snt11.west);
  \node (sw1) [below=-0.3ex of sw1a.south] {\strut was};
  \node (sw2) [below=-0.3ex of snt20.south] {\strut named};
  \node (sw3a) at ($(snt30.east)!0.5!(snt32.west)$) {\strut};
  \draw (snt30.east) -- (snt31.west);
  \draw (snt31.east) -- (snt32.west);
  \node (sw3) [below=-0.3ex of sw3a.south] {\strut CEO};
  \node (sw4) [below=-0.3ex of snt40.south] {\strut today};

  \draw [structural] (snt10.north) to [out=135,in=45] (sntR0.north);
  \draw [structural] (snt40.north) to [out=135,in=45] (snt20.north);
  \draw [trace] (snt00.north) to [out=45,in=135] (snt31.north);

\end{tikzpicture}}
  \caption{\label{fig:not-1ec}
    Not One-Endpoint Crossing
  }
\end{subfigure}
\vspace{-3mm}
\caption{Examples of problematic graph structured syntactic phenomena before our head rule changes.}
\end{figure}

To construct the spines, we lexicalize with head rules that consider the type of each non-terminal and its children.
Different heads often represent more syntactic or semantic aspects of the phrase.
For trees, all head rules generate valid structures.
For graphs, head rules influence the creation of two problematic structures:

\tightparagraph{Cycles}
These arise when the head chosen for a phrase is also an argument of another word in the phrase.
Figure~\ref{fig:cycle} shows a cycle between \emph{which} and \emph{proposed}.
We resolve this by changing the head of an SBAR to be an S rather than a Wh-noun phrase.

\tightparagraph{One-Endpoint Crossing Violations}
Figure~\ref{fig:not-1ec} shows an example, with the trace from \emph{CEO} to \emph{Page} crossing two edges with no endpoints in common.
We resolve this case by changing the head for VPs to be a child VP rather than an auxiliary.



\section{Results} \label{sec:impl}

\tightparagraph{Algorithm Coverage}
In Table~\ref{tab:coverage} we show the impact of design decisions for our representation.
The percentages indicate how many sentences in the training set are completely recoverable by our algorithm.
Each row shows the outcome of an addition to the previous row, starting from no traces at all, going to our representation with the head rules of \textcite{cck}, then changing the head rules, reversing null-null edges, and changing the target of edges in parallel constructions.
The largest gain comes from changing the head rules, which is unsurprising since \textcite{cck}'s rules were designed for trees (any set of rules form valid structures for trees).

\tightparagraph{Problematic Structures}
Of the sentences we do not cover, $54\%$ contain a cycle, $45\%$ contain a \oneEC violation, and $1\%$ contain both.
To understand these problematic sentences, we manually inspected a random sample of twenty parses that contained a cycle and twenty parses with a \oneEC violation (these forty are $6\%$ of all problematic parses, enough to identify the key remaining challenges).

For the cycles, eleven cases related to sentences containing variations of NP~\emph{said} interposed between two parts of a single quote.
A cycle was present because the top node of the parse was co-indexed with a null argument of \emph{said} while \emph{said} was an argument of the head word of the quote.
The remaining cases were all instances of pseudo-attachment, which the treebank uses to show that non-adjacent constituents are related \parencite{ptb-guide}.
These cases were split between use of Expletive (5) and Interpret Constituent Here (4) traces.

It was more difficult to determine trends for cases where the parse structure has a \oneEC violation.
The same three cases, Expletive, Interpret Constituent Here, and NP \emph{said} accounted for half of the issues.

\subsection{Implementation}
We implemented a parser with a first-order model using our algorithm and representation.
Code for the parser, for conversion to and from our representation, and for our metrics is available\footnote{
  \url{https://github.com/jkkummerfeld/1ec-graph-parser}
}.
Our parser uses a linear discriminative model, with features based on \textcite{McDonald-etal:2005:Proj}.
We train with an online primal subgradient approach \parencite{Ratliff:2007} as described by \textcite{Kummerfeld-etal:2015:EMNLP}, with parallel lock-free sparse updates.

\begin{table}
\small
\centering
  \vspace{2mm}
  \begin{tabular}{lrr}
    \hline
    & \multicolumn{2}{c}{Coverage (\%)} \\
    Representation & Sentences & Edges \\
    \hline
    \hline
    Projective trees, no nulls & 26.59 & 96.27 \\
    Projective trees, with nulls & 43.85 & 96.27 \\
    Projective graphs & 50.60 & 96.67 \\
    One-EC graphs & 71.84 & 98.31 \\
    + Head rule changes & 92.35 & 99.23 \\
    + Null reversal & 97.02 & 99.45 \\
    + Parallel construction shift & 97.31 & 99.49 \\
    \hline
  \end{tabular}
  \caption{ \label{tab:coverage}
    Training set coverage for different representations.
    One-EC graphs uses our representation, but with the head rules from \textcite{cck}.
    For the edge results, we only exclude edges necessary to make each parse representable (\myeg excluding only one edge in a cycle and counting the rest).
  }
\end{table}

\tightparagraph{Loss Function}
We use a weighted Hamming distance for loss-augmented decoding, as it can be efficiently decomposed within our dynamic program.
Calculating the loss for incorrect spines and extra edges is easy.
For missing edges, we add when a deduction rule joins two spans that cover an end of the edge, since if it does not exist in one of those items it is not going to be created in future.
To avoid double counting we subtract when combining two halves that contain the two ends of a gold edge\footnote{
One alternative is to count half of it on each end, removing the need for subtraction later.
Another is to add it during the combination step.
}.

\tightparagraph{Inside--Outside Calculations}
Assigning scores to edges is simple, as they are introduced in a single item in the derivation.
Spines must be introduced in multiple items (left, right, and external positions) and must be assigned a score in every case to avoid ties in beams.
We add the score every time the spine is introduced and then subtract when two items with a spine in common are combined.

\tightparagraph{Algorithm rule pruning}
Many \oneEC structures are not seen in our data.
We keep only the rules used in gold training parses, reducing the set of 49,292 from the general algorithm to 627 (including rules for both adding arcs and combining items).
Almost every template in Algorithm~\ref{alg:rules} generates some unnecessary rules, and no items of type $B$ are needed.
The remaining rules still have high coverage of the development set, missing only 15 rules, each applied once (out of 78,692 rule applications).
By pruning in this way, we are considering the intersection of \oneEC graphs and the true space of structures used in language.


\tightparagraph{Chart Pruning}
To improve speed we use beams and cube pruning \parencite{Chiang:2007}, discarding items based on their Viterbi inside score.
We divide each beam into sub-beams based on aspects of the state.
This ensures diversity and enables consideration of only compatible items during binary and ternary compositions.

\tightparagraph{Coarse to Fine Pruning}
Rather than parsing immediately with the full model we use several passes with progressively richer structure \parencite{Goodman:1997}:
(1) Projective parsing without traces or spines, and simultaneously a trace classifier,
(2) Non-projective parsing without spines, and simultaneously a spine classifier,
(3) Full structure parsing.
Each pass prunes using parse max-marginals and classifier scores, tuned on the development set.
The third pass also prunes spines that are not consistent with any unpruned edge from the second pass.
For the spine classifier we use a bidirectional LSTM tagger, implemented in DyNet \parencite{dynet}.

\tightparagraph{Speed}
Parsing took an average of $8.6$ seconds per sentence for graph parsing and $0.5$ seconds when the parser is restricted to trees\footnote{
  Using a single core of an Amazon EC2 m4.2xlarge instance (2.4 GHz Xeon CPU and 32 Gb of RAM).
}.
Our algorithm is also amenable to methods such as semi-supervised and adaptive supertagging, which can improve the speed of a parser after training \parencite{Lewis-Steedman:2014,Kummerfeld-Roesner-Dawborn-Haggerty-Curran-Clark:2010:ACL}.

\tightparagraph{Tree Accuracy}
On the standard tree-metric, we score $88.1$.
Using the same non-gold POS tags as input, \textcite{cck} score $90.9$, probably due to their second-order features and head rules tuned for performance\footnote{
  Previous work has shown that the choice of head can significantly impact accuracy \parencite{schwartz-abend-rappoport:2012:PAPERS}.
}.
Shifting to use their head rules, we score $88.9$.
Second-order features could be added to our model through the use of forest reranking, an improvement that would be orthogonal to this paper's contributions.

We can also evaluate on spines and edges.
Since their system produces regular \ptb trees, we convert its output to our representation and compare its results with our system using their head rules.
We see slightly lower accuracy for our system on both spines (94.0 \myvs 94.3) and edges (90.4 \myvs 91.1).

\begin{table}
\small
\centering
  \vspace{2mm}
  \begin{tabular}{lrrr}
    \hline
    System & P & R & F \\
    \hline
    \hline
    \multicolumn{4}{c}{Null Elements Only} \\
    \textcite{Johnson:2002} & 85 & 74 & 79 \\
    \textcite{hayashi-nagata:2016} & 90.3 & 81.7 & 85.8 \\
    \textcite{kato-matsubara:2016} & 88.5 & 82.1 & 85.2 \\
    This work & 89.5 & 81.6 & 85.4 \\
    \hline
    \multicolumn{4}{c}{Null Elements and Co-indexation} \\
    \textcite{Johnson:2002} & 73 & 63 & 68 \\
    \textcite{kato-matsubara:2016} & 81.2 & 74.7 & 77.8 \\
    This work & 74.3 & 67.3 & 70.6 \\
    \hline
  \end{tabular}
  \caption{ \label{tab:accuracy}
    Accuracy on section 23 using Johnson's metric.
  }
\end{table}

\newcommand{\kmcite}{K\&M \parencite*{kato-matsubara:2016}\xspace}

\tightparagraph{Trace Accuracy}
Table~\ref{tab:accuracy} shows results using \textcite{Johnson:2002}'s trace metric.
Our parser is competitive with previous work that has highly-engineered models: Johnson's system has complex non-local features on tree fragments, and similarly \textcite[K\&M][]{kato-matsubara:2016} consider complete items in the stack of their transition-based parser.
On co-indexation our results fall between Johnson and K\&M.
Converting to our representation, our parser has higher precision than K\&M on trace edges ($84.1$ \myvs $78.1$) but lower recall ($59.5$ \myvs $71.3$).
One modeling challenge we observed is class imbalance: of the many places a trace could be added, only a small number are correct, and so our model tends to be conservative (as shown by the P/R tradeoff).

\section{Conclusion}

We propose a representation and algorithm that cover $97.3\%$ of graph structures in the \ptb.
Our algorithm is $O(n^4)$, uniquely decomposes parses, and enforces the property that parses are composed of a core tree with additional traces and null elements.
A proof of concept parser shows that our algorithm can be used to parse and recover traces.

\section*{Acknowledgments}

Thank you to Greg Durrett for advice on parser implementation and debugging, and to the anonymous reviewers for their helpful feedback.
This research was partially supported by a General Sir John Monash Fellowship and the Office of Naval Research under MURI Grant No.\@\xspace N000140911081.

\citetrackertrue\pagetrackertrue\backtrackertrue

\newrefcontext[sorting=nyt]
\printbibliography

\end{document}